\documentclass[letterpaper, 10 pt, conference]{ieeeconf}  

\IEEEoverridecommandlockouts                              

\pdfoutput=1
\usepackage{amsmath}
\usepackage{amsfonts}
\usepackage{graphicx}
\usepackage{colortbl}
\usepackage{cite}

\newtheorem{theorem}{Theorem}[section]

\newtheorem{lemma}[theorem]{Lemma}
\newtheorem{corollary}[theorem]{Corollary}

\newtheorem{assumption}[theorem]{Assumption}
\newtheorem{remark}[theorem]{Remark}
\newtheorem{example}{Example}

\makeatletter
\let\NAT@parse\undefined
\makeatother
\usepackage{hyperref}       

\title{\LARGE \bf
On the Benefits of Leveraging Structural Information in Planning Over the Learned Model
}

\author{Jiajun Shen$^{1}$, Kananart Kuwaranancharoen$^{1}$, Raid Ayoub$^{2}$, Pietro Mercati$^{2}$ and Shreyas Sundaram$^{1}$
\thanks{*This work was supported by a grant from Intel Corporation.}
\thanks{$^{1}$Jiajun Shen, Kananart Kuwaranancharoen and Shreyas Sundaram are with the Elmore Family School of Electrical Computer Engineering, Purdue University, West Lafayette, IL 47906, USA
        {\tt\small shen590, kkuwaran, sundara2@purdue.edu}}%
\thanks{$^{2}$Raid Ayoub and Pietro Mercati are with Intel Corporation, 2111 NE 25th Ave., Hillsboro, OR 97124, USA
        {\tt\small pietro.mercati, raid.ayoub@intel.com}}%
}

\begin{document}

\maketitle
\thispagestyle{empty}
\pagestyle{empty}

\begin{abstract}

Model-based Reinforcement Learning (RL) integrates learning and planning and has received increasing attention in recent years. However, learning the model can incur a significant cost (in terms of sample complexity), due to the need to obtain a sufficient number of samples for each state-action pair. In this paper, we investigate the benefits of leveraging structural information about the system in terms of reducing the sample complexity.  Specifically, we consider the setting where the transition probability matrix is a known function of a number of structural parameters, whose values are initially unknown. We then consider the problem of estimating those parameters based on the interactions with the environment. We characterize the difference between the Q estimates and the optimal Q value as a function of the number of samples. Our analysis shows that there can be a significant saving in sample complexity by leveraging structural information about the model. We illustrate the findings by considering several problems including controlling a queuing system with heterogeneous servers, and seeking an optimal path in a stochastic windy gridworld.
\end{abstract}

\section{INTRODUCTION}

Recent years have witnessed the great success of model-based reinforcement learning (RL) in various fields (e.g., games \cite{Silver17}, queuing systems \cite{Liu19}, and robotics \cite{Polydoros17}). Benefiting from the integration of planning and learning, model-based RL has advantages in terms of better data efficiency, interpretability, higher asymptotic performance (optimality/cumulative reward) and transfer learning. 

In many practical scenarios, it is challenging to learn the exact Markov decision process (MDP) model from data. However, if one has some prior knowledge of certain structural properties of the system (especially for some customized or specific applications), it may be possible to leverage that structure (or ``grey-box model'') to reduce the sample complexity of learning. This motivates our main focus in this paper. Specifically, our goal is to \textit{provide a sample complexity bound on achieving a near-optimal estimate of the action-value function (i.e., Q function) when leveraging structural information about the model}.

\subsection*{Related Work}
A well known minimax probably approximately correct (PAC) bound is proposed in \cite{gheshlaghi2013minimax} for finding an $\epsilon$-optimal estimate of the action-value function (as well as $\epsilon$-optimal policy) with high probability (w.h.p.). Many other papers have studied both model-based and model-free RL (\cite{sidford2018variance}, \cite{even2006action}, \cite{azar2011reinforcement}, \cite{wiering2012reinforcement}) providing similar tight bounds. Among these results, when using model-based RL approaches, it is typical to assume access to a generative model for the underlying MDP, and that each element of the transition matrix is being estimated independently (i.e., element-wise estimation). Such assumptions make the corresponding bounds linearly depend on the size of state-action space (i.e., $N := |\mathcal{S}||\mathcal{A}|$, where $\mathcal{S}$ and $\mathcal{A}$ are the state and action space of the MDP respectively).

To solve this issue, \cite{wang2021sample} considered some structural assumptions on the transition dynamics based on the linear transition model proposed in \cite{yang2019sample}. The key idea is to express the transition kernel as a convex combination form, i.e., for every state-action pair $(s,a) \in \mathcal{S} \times \mathcal{A}$,
$$ P(\cdot|s,a)=\sum_{i:(s_i,a_i) \in \mathcal{K}} \lambda_i(s,a) P(\cdot|s_i,a_i),
$$ where $\sum_{i=1}^K \lambda_i(s,a)=1$, $\lambda_i(s,a) \geq 0$, and $\mathcal{K}$ indicates a certain set of anchor state-action pairs $\mathcal{K} \subset \mathcal{S} \times \mathcal{A}$. Using the similar ``leave-one-out'' technique as in \cite{agarwal2020model} and under several assumptions on the generative model, known anchor state-action pairs, feature mapping, and all weighting factors, the authors achieved a better sample complexity bound for $\epsilon$-optimal Q estimation by replacing the term $N$ with $K:=|\mathcal{K}|$. However, $K$ may scale linearly with $|\mathcal{S}|$ and therefore still suffer the curse of dimensionality as the classic bound in \cite{gheshlaghi2013minimax}. In addition, the sampling process in their algorithm (Algorithm~1 in Section 3.1 of \cite{wang2021sample}) requires iterating through every single element (i.e., anchor state-action pair) within the set $\mathcal{K}$, and assumes access to a generative model as well.

Instead of using a generative model, some other approaches (e.g., PAC-MDP \cite{kakade2003sample, szita2010model}, upper-confidence-bound reinforcement learning (UCRL) \cite{auer2008near}, robust MDP \cite{yang2022toward}, and block MDP \cite{modi2021model, zhang2022efficient}) take the exploration policy into account when analyzing the sample complexity. These methods often require specific designs for the exploration policy. Besides, the PAC-MDP method derived loose bounds due to bounding the estimation error in terms of the largest possible reward \cite{gheshlaghi2013minimax}. In parallel to \cite{gheshlaghi2013minimax}, tight upper and lower bounds of the same order were independently derived for the UCRL-based algorithm in \cite{lattimore2012pac} under strong assumptions on the transition model (such as only two states being accessible for any state-action pair). The robust MDP method \cite{yang2022toward} focused on capturing the perturbation of reward and transition probability by some specifically designed uncertainty set (or ambiguity set). The block MDP method \cite{modi2021model, zhang2022efficient} assumed the existence of an unknown mapping from the observed state space to the so-called latent state space (i.e., each observation state is generated by only one latent state). Strong assumptions were required on either reachability \cite{modi2021model} of each latent state or finite candidate set (e.g., feature class set) and episodic tasks \cite{zhang2022efficient}. 

Our paper deviates from the lines of most previous work by adopting a high-level grey-box framework where the entries of the probability transition matrix are not fully independent, but instead possess some known structure as we will explicitly show later. Our paper differs from \cite{wang2021sample} in the following key aspects.

(1) Our sample complexity bound is a function of the minimum amount of information about each structural parameter in any given exploration policy. Thus, we do not restrict attention to generative models (as in \cite{wang2021sample}) or specific exploration policies for sample collection. We show the better sample efficiency of our approach when using a generative model theoretically, and experimentally evaluate some simple heuristic exploration policies.

(2) In \cite{wang2021sample}, authors considered convex mapping between transition probabilities. However, motivated by practical systems (e.g., queuing system), we consider a more general mapping between transition probabilities and system structures. We show that the mapping considered in \cite{wang2021sample} can be viewed as a special case of ours when the system structure information is extremely limited.


To account for the aforementioned structural knowledge, we adopt a two-stage (off-policy and offline) framework, i.e., (1) collect samples by interacting with the environment, and (2) implement model estimation by leveraging the collected samples and structural knowledge, and plan over the estimated model to find the optimal action-value function. Indeed, while it is not surprising that leveraging structural information will speed up learning, our theoretical and experimental results \textit{quantify} the benefits of leveraging structural knowledge about the probability transition matrix.

\section{Preliminaries}

\textbf{Notations:} All norms $\| \cdot \|$, unless otherwise specified, are infinity norms. We denote $\boldsymbol{1}$ as the vector of all ones.

A discounted MDP is a quintuple $M = \langle \mathcal{S}, \mathcal{A}, P, R, \gamma \rangle$, where $\mathcal{S}$ is the state space, $\mathcal{A}$ is the action space,
$P: \mathcal{S} \times \mathcal{A} \times \mathcal{S} \rightarrow [0,1]$ is the transition probability function, $R: \mathcal{S} \times \mathcal{A} \rightarrow \mathbb{R}$ is the reward function, and $\gamma \in (0, 1)$ is the discount factor. At each time-step $t$, the agent receives reward $r_{t}=R(s_t,a_t)$ after taking action $a_t$ at state $s_t$.\footnote{For simplicity, we assume that $R(s, a)$ is a deterministic function of state-action pairs $(s, a)$.}  
Given a stationary and deterministic Markovian policy $\pi: \mathcal{S} \rightarrow \mathcal{A}$, at each time-step $t$, the control action $a_t \in \mathcal{A}$ is given by $a_t = \pi (s_t)$ where $s_t \in \mathcal{S}$. The state-value function for a given state $s$, under policy $\pi$ and MDP $M$, is denoted as  
\begin{equation*}
    V^{\pi} (s) := \mathbb{E} \Big[ \sum_{i=0}^\infty \gamma^i r_{t+i} \; | \; s_t = s \Big].
\end{equation*}
The action-value function for a given state-action pair $(s, a)$ and under policy $\pi$ is defined as
\begin{equation*}
    Q^{\pi}(s, a) 
    := R (s, a) + \gamma \sum_{s' \in \mathcal{S}} P(s' | s, a)  V^{\pi} (s').
\end{equation*}
We use the notation $\mathcal{Z}$ for the joint state-action space $\mathcal{S} \times \mathcal{A}$. We also use the notations $z$ and $\beta$ for the state-action pair $(s, a) \in \mathcal{Z}$ and $1/(1 - \gamma)$, respectively.

\section{Problem Formulation}

We consider an MDP $M$, where the state space $\mathcal{S}$, action space $\mathcal{A}$, reward function $R$, and discount factor $\gamma$ are known; however, we assume that the dynamics (transition probability matrix $P$) is initially unknown. 
We assume that at each time-step $t \in \mathbb{N}$, we obtain a new state transition $(s_t, a_t, s'_t)$, with the associated reward. Thus, at time-step $k \in \mathbb{N}$, we have a dataset containing state transitions $\mathcal{D}_k = \{ (s_t, a_t, s'_t) \}_{t=1}^k$.
In order to derive near-optimal Q estimate, we seek to first approximate the transition probability matrix $P$ by an estimate, $P_k$, based on the collected samples up to time-step $k$. We define an estimated MDP $M_k$ induced by $P_k$ as $M_k = \langle \mathcal{S}, \mathcal{A}, P_k, R, \gamma \rangle$, and the optimal action-value functions under $M$ and $M_k$ are denoted as $Q^*$ and $Q_k^*$, respectively.

We make the following assumption on the state and action spaces, and reward function.

\begin{assumption}
We assume $\mathcal{S}$ and $\mathcal{A}$ (and consequently, $\mathcal{Z}$) are finite sets with cardinalities $| \mathcal{S} |$, $| \mathcal{A} |$, and $N$ respectively. We also assume that the reward function $R(s, a)$ takes values from the interval $[0, 1]$.\footnote{The results hold if the rewards take values from some interval $[r_{\text{min}}, r_{\text{max}}]$ instead of $[0, 1]$, in which case the bounds scale with the factor $(r_{\text{max}} - r_{\text{min}})^2$.}
\end{assumption}

\section{Estimation via Structural Parameters and Evaluation Error Analysis}

\label{section:Estimation via Structural Parameters and Evaluation Error Analysis}

In this section, we introduce a methodology to estimate the model (i.e., transition probabilities $P$) based on structural information defined below, and then characterize the performance difference as captured by $\| Q^* - Q_k^*\|$.

We make the following assumption about the transition probabilities.

\begin{assumption}  \label{assumption: structural parameters}
The true transition probability function $P$ for MDP $M$ can be represented as a known function of $m$ unknown structural parameters $\{ \mu_i^* \}_{i=1}^m \subset \mathbb{R}$, i.e., for each $z \in \mathcal{Z}$ and $s' \in \mathcal{S}$, there exists a known function $f^{s'}_{z}: \mathbb{R}^m \to \mathbb{R}$ such that
\begin{equation*}
    P(s' | z) = f^{s'}_{z} (\boldsymbol{\mu}^*),
\end{equation*}
where $\boldsymbol{\mu}^* = (\mu_1^*, \mu_2^*, \ldots, \mu_m^*)$ lies in a compact subset of $\mathbb{R}^m$. Furthermore, for all $(z,s') \in \mathcal{Z} \times \mathcal{S}$ such that $P(s'|z) = 0$, we have $f_z^{s'}(\boldsymbol{\mu}) = 0$ for all $\boldsymbol{\mu}$.
\end{assumption}

Under the above assumption, the {\it form} of the transition probabilities (i.e., each of the functions $f^{s'}_{z}$) is known; however, the {\it values} of the structural parameters $\boldsymbol{\mu}^*$ are not known \footnote{We will later provide examples of applications where this assumption holds.}. Thus, by estimating the values of the parameters from data, one can then obtain an estimate of the transition probabilities.\footnote{Note that the case of having no structural information can be captured as a special case of this formulation by setting each element of $P$ to be a different structural parameter.}  In order to estimate the structural parameters $\boldsymbol{\mu}^*$, we now define the subset of tuples $(s,a,s')$ that provides information about each parameter.

\begin{assumption} \label{assumption: subset contains info}
For $i$-th structural parameter $\mu_i$, there exists a corresponding non-empty subset $\mathcal{U}_i \subseteq \mathcal{S} \times \mathcal{A}$ of the state-action pairs that provides information about its true value $\mu_i^*$. In particular, we assume there exists a random variable $X_i: \mathcal{U}_i \to \mathbb{R}$ (where the randomness is induced by the distribution of $s'$ given $s$ and $a$) such that $\mathbb{E} [ X_i ] = \mu_i^*$, and $X_i$ has a finite variance $\sigma_{\mu_i}^2$.
\end{assumption}

As defined in the previous section, for all $k \in \mathbb{N}$, let $\mathcal{D}_k := \{(s_i,a_i,s'_i)\}_{i=1}^k$ be the set of all 3-tuples obtained up to time-step $k$.  For all $i \in \{1, 2, \ldots, m\}$, we define the set of transitions that can be used to estimate the structural parameter $\mu_i$ up to time $k$ as 
$\mathcal{D}_{k, i} := \mathcal{D}_k \cap (\mathcal{U}_i \times \mathcal{S}_i')$ ($\mathcal{S}_i' \subseteq \mathcal{S}$ denotes all possible next states given $\mathcal{U}_i$),
and the corresponding number of transitions in such set as
$n_{k, i} := | \mathcal{D}_{k, i} |$.
In other words, based on the transitions seen up to time-step $k$, $n_{k,i}$ of those transitions provide information about the structural parameter $\mu_i$.  We define 
\begin{equation}
    n_k := \min_{i \in \{ 1, 2, \ldots, m \}} n_{k,i}
    \label{def: min_samples}
\end{equation}
to represent the minimum amount of information we have received about any structural parameter up to time-step $k$.  

For the $i$-th structural parameter at time $k \in \mathbb{N}$, if $n_{k,i} > 0$, we consider the estimator as
\begin{equation}
    \hat{\mu}_{k,i} = \frac{1}{n_{k,i}} \sum_{t \in T_{k,i}} X_{t,i},
    \label{eqn: estimator}
\end{equation}
where $X_{t,i}$ is the sample obtained at time-step $t$ and is relevant to $i$-th structural parameter, and
$T_{k, i}$ is the set of time-steps such that the transitions from that time-step  contain  information about the structural parameter $\mu_i$, i.e.,
\begin{equation*}
    T_{k, i} := \big\{ t \in \mathbb{N} : (s_t, a_t, s'_t) \in \mathcal{D}_{k,i} \big\}.
\end{equation*}
From the above definitions, we have that $| T_{k, i} | = n_{k, i}$ and $\hat{\mu}_{k,i}$ is an unbiased estimator of $\mu_i^*$, i.e., $\mathbb{E} [\hat{\mu}_{k, i}] = \mu_i^*$.
At time-step $k$, the corresponding transition probability function is computed by $P_k (s' | z) = f^{s'}_{z} ( \hat{\boldsymbol{\mu}}_k )$ for all $z \in \mathcal{Z}$ and $s' \in \mathcal{S}$ where $\hat{\boldsymbol{{\mu}}}_k = (\hat{\mu}_{k,1}, \hat{\mu}_{k,2}, \ldots, \hat{\mu}_{k,m})$.
Before analyzing the problem, we assume a general property of the reconstruction functions $f^{s'}_{z}$ as follows.

\begin{assumption}  \label{assumption: Lipschitz}

For each pair $(z,s') \in \mathcal{Z} \times \mathcal{S}$, there exists a constant $L_z^{s'} \in \mathbb{R}$ such that the function $f_z^{s'}: \mathbb{R}^m \rightarrow \mathbb{R}$ is $L_z^{s'}$-Lipschitz continuous in the entire structural parameter vector space, i.e.,
\begin{equation*}
    | f^{s'}_{z} (\boldsymbol{\mu}_1) - f^{s'}_{z} (\boldsymbol{\mu}_2) | 
    \leq L_z^{s'} \| \boldsymbol{\mu}_1 - \boldsymbol{\mu}_2 \|_2
\end{equation*}
for all $\boldsymbol{\mu}_1, \boldsymbol{\mu}_2$ in the parameter space (a compact subset of $\mathbb{R}^m$).
\end{assumption}

For all $(z,s') \in \mathcal{Z} \times \mathcal{S}$ such that $P(s'|z) \neq 0$, we define a uniform constant $L$ such that $L_z^{s'} \leq L P(s'|z)$. Since $f_z^{s'}(\cdot) = 0$ when $P(s'|z) = 0$ (by Assumption~\ref{assumption: structural parameters}), without loss of generality, we can rewrite the inequality in Assumption~\ref{assumption: Lipschitz} as
\begin{equation*}
    | f^{s'}_{z} (\boldsymbol{\mu}_1) - f^{s'}_{z} (\boldsymbol{\mu}_2) | 
    \leq L P(s'|z) \| \boldsymbol{\mu}_1 - \boldsymbol{\mu}_2 \|_2.
\end{equation*}


Recall that $Q^*$ and $Q_k^*$ are the optimal action-value functions corresponding to MDPs $M = \langle \mathcal{S}, \mathcal{A}, P, R, \gamma \rangle$ and $M_k = \langle \mathcal{S}, \mathcal{A}, P_k, R, \gamma \rangle$, respectively. In practice, given the MDP $M_k$, several standard algorithms (e.g. policy iteration \cite{howard1960dynamic} and value iteration \cite{bellman1966dynamic}) are guaranteed to asymptotically converge to the optimal value $Q^*_k$. In this paper, we disregard the computation of the optimal values and focus on the gap between the optimal action-value functions for the true and estimated models.

Before stating the results, we first introduce several notations as follows. We use $\hat{V}^{\pi}$ to denote the empirical state-value function of a policy $\pi$ and an estimate, $P_k$. The optimal policy $\pi^*$ (resp. $\hat{\pi}^*_k$) is the policy which attains $V^*$ (resp. $V_k^*$) under the model $P$ (resp. $P_k$).

The right-linear operators $P^{\pi} \cdot$ and $P \cdot$ are defined as
\begin{align*}
    (P^{\pi} Q)(z) &:= \sum_{y \in \mathcal{S}} P(y | z) Q(y, \pi(y) ) 
    \quad \text{for all} \; z \in \mathcal{Z}, \\
    (P V^{\pi})(z) &:= \sum_{y \in \mathcal{S}} P(y | z) V^{\pi}(y)
    \quad \text{for all} \; z \in \mathcal{Z}.
\end{align*}
For any policy $\pi$, we also define the operator $(P^{\pi})^i \cdot$ as
\begin{equation*}
    (P^{\pi})^i Q(z) := \underbrace{P^{\pi} \cdots P^{\pi}}_{i \; \text{times}} Q(z)
    \quad \text{for all} \; i \in \mathbb{N} \;\; \text{and} \; z \in \mathcal{Z}.
\end{equation*}
Based on the above definition, the operator $(I - \gamma P^\pi)^{-1} \cdot$ is defined as 
\begin{equation*}
    (I - \gamma P^\pi)^{-1} Q(z) 
    := \sum_{i=0}^\infty (\gamma P^{\pi})^i Q(z)
    \quad \text{for all} \; z \in \mathcal{Z}.
\end{equation*}

For any real-valued function $f: \mathcal{Y} \to \mathbb{R}$, where $\mathcal{Y}$ is a finite set, we define the variance of $f$ under the probability distribution $\rho$ on $\mathcal{Y}$ as 
\begin{equation*}
    \mathbb{V}_{y \sim \rho} (f(y)) := \mathbb{E}_{y \sim \rho} | f(y) - \mathbb{E}_{y \sim \rho} (f(y)) |^2.
\end{equation*}
Based on this definition, we define the empirical variance of the state-value function $V^\pi$ as
\begin{equation*}
    \hat{\sigma}_{V^\pi} (z) := \gamma^2 \mathbb{V}_{y \sim P_k ( \cdot | z)} ( V^\pi (y) ) 
    \quad \text{for all} \quad z \in \mathcal{Z}.
\end{equation*}

Recall that $N = |\mathcal{S}| |\mathcal{A}|$, $\beta = \frac{1}{1 - \gamma}$, $n_k$ is defined in \eqref{def: min_samples}, $\hat{\boldsymbol{\mu}}_k$ is the vector whose components are determined by \eqref{eqn: estimator}, and $\boldsymbol{\mu}^*$ is defined in Assumption~\ref{assumption: structural parameters}. Moreover, we define 
\begin{equation*}
    \sigma_{\boldsymbol{\mu}} := \sum_{i=1}^m \sigma_{\mu_i},
\end{equation*}
where $\sigma_{\mu_i}$ are defined in Assumption~\ref{assumption: subset contains info}.

Lemma~\ref{lem: mu bound} and \ref{lem: pv concentration} give an error bound on $\hat{\boldsymbol{\mu}}_k$ and $P_k V^*$, respectively, from their true value. 

\begin{lemma}  \label{lem: mu bound}
If Assumption~\ref{assumption: subset contains info} holds, then for all $k \in \mathbb{N}$ such that $n_k > 0$, we have
\begin{equation*}
    \mathbb{E} \big[ \| \hat{\boldsymbol{\mu}}_k - \boldsymbol{\mu}^* \|_2 \big] 
    \leq \frac{\sigma_{\boldsymbol{\mu}}}{\sqrt{n_k}}.
\end{equation*}
\end{lemma}

\begin{proof}
Consider the quantity $\big( \mathbb{E} \big[ \| \hat{\boldsymbol{\mu}}_k - \boldsymbol{\mu}^* \|_2 \big] \big)^2$ as follows. From Jensen's inequality, we have
\begin{equation}
    \big( \mathbb{E} \big[ \| \hat{\boldsymbol{\mu}}_k - \boldsymbol{\mu}^* \|_2 \big] \big)^2
    \leq \mathbb{E} \big[ \| \hat{\boldsymbol{\mu}}_k - \boldsymbol{\mu}^* \|_2^2 \big]
    = \sum_{i=1}^m \mathbb{E} \big[ (\hat{\mu}_{k,i} - \mu^*_i)^2 \big].
    \label{eqn: expected square}
\end{equation}
From \eqref{eqn: estimator} and Assumption~\ref{assumption: subset contains info}, we have
\begin{equation*}
    \mathbb{E}[ \hat{\mu}_{k,i} ] 
    = \mathbb{E} \bigg[ \frac{1}{n_{k,i}} \sum_{t \in T_{k,i}} X_{t,i} \bigg]
    = \mu^*_i.
\end{equation*}
Using the above equation, we can write
\begin{align}
    \mathbb{E} \big[ (\hat{\mu}_{k,i} - \mu^*_i)^2 \big]
    &= \mathbb{V} (\hat{\mu}_{k,i}) \nonumber \\
    &= \mathbb{V} \bigg( \frac{1}{n_{k,i}} \sum_{j \in T_{k,i}} X_{j,i} \bigg) \nonumber \\
    &= \frac{1}{n_{k,i}} \mathbb{V} (X_i).
    \label{eqn: square expected}
\end{align}
Combining \eqref{eqn: expected square} and \eqref{eqn: square expected} together and taking square root of both sides yield
\begin{equation*}
    \mathbb{E} \big[ \| \hat{\boldsymbol{\mu}}_k - \boldsymbol{\mu}^* \|_2 \big]
    \leq \frac{1}{\sqrt{n_{k}}} \sqrt{ \sum_{i=1}^m \mathbb{V} (X_i) } 
    \leq \frac{1}{\sqrt{n_{k}}} \sum_{i=1}^m \sigma_{\mu_i},
\end{equation*}
where the first inequality is due to $n_k \leq n_{k,i}$ for all $i \in \{1,2,\ldots,m\}$.
\end{proof}

\begin{lemma}  \label{lem: pv concentration}
If Assumptions~\ref{assumption: structural parameters}, \ref{assumption: subset contains info} and \ref{assumption: Lipschitz} hold, then for all $\delta \in (0,1)$ and $k \in \mathbb{N}$ such that $n_k > 0$, with probability at least $1 - \delta$, we have
\begin{align}
    \gamma [ P V^* - P_k V^* ]
    \leq c_{pv} \sqrt{ \hat{\sigma}_{\hat{V}^{\pi^*} }} + b_{pv} \boldsymbol{1},
    \label{eqn: PV upper bound} \\
    \gamma [ P V^* - P_k V^* ]  
    \geq - c_{pv} \sqrt{ \hat{\sigma}_{\hat{V}^{\pi^*} }} - b_{pv} \boldsymbol{1},
    \label{eqn: PV lower bound} 
\end{align}
where $c_{pv} = \big( \frac{2}{n_k} \log \frac{2N}{\delta} \big)^{\frac{1}{2}}$ and
\begin{equation*}
    b_{pv} = \frac{ \gamma \beta L \sigma_{\boldsymbol{\mu}}}{\sqrt{n_k}} 
    + \bigg( \frac{ 5 (\gamma \beta)^{\frac{4}{3}} }{n_k} \log \frac{6N}{\delta} \bigg)^{\frac{3}{4}} 
    + \frac{3 \beta^2 }{n_k} \log \frac{12 N}{\delta}.
\end{equation*}
\end{lemma}

\begin{proof}
We can rewrite the term $[ P V^* - P_k V^* ] (z)$ as follows:
\begin{multline}
    \gamma [ P V^* - P_k V^* ] (z)
    = \gamma \big[ P V^* - \mathbb{E}[ P_k V^* ] \big] (z) \\
    + \gamma \big[ \mathbb{E}[ P_k V^* ] - P_k V^* \big] (z).
    \label{eqn: PV expand}
\end{multline}
Recall that $P(s' | z) = f^{s'}_{z} (\boldsymbol{\mu}^*)$ and $P_k (s' | z) = f^{s'}_{z} ( \hat{\boldsymbol{\mu}}_k )$ for all $z \in \mathcal{Z}$ and $s' \in \mathcal{S}$.
The first term of \eqref{eqn: PV expand} can be bounded as follows:
\begin{align}
    &\big[ P V^* - \mathbb{E}[ P_k V^* ] \big] (z) \nonumber \\
    &= \mathbb{E} \bigg[ \sum_{s' \in \mathcal{S}} P(s'|z) V^*(s') - \sum_{s' \in \mathcal{S}} P_k(s'|z) V^*(s') \bigg] \nonumber \\
    &= \mathbb{E} \big[ \big( \boldsymbol{f}( \boldsymbol{\mu}^* ) - \boldsymbol{f} ( \hat{\boldsymbol{\mu}}_k ) \big)^T V^* \big] (z) \nonumber \\
    &\leq \mathbb{E} \big[ \| \boldsymbol{f}( \boldsymbol{\mu}^* ) - \boldsymbol{f} ( \hat{\boldsymbol{\mu}}_k ) \|_1 \; \| V^* \|_{\infty} \big] (z),
    \label{eqn: PV first term 1}
\end{align}
where $\boldsymbol{f}( \boldsymbol{\mu} ) (z) = [ f^{s'}_z ( \boldsymbol{\mu} ) ]_{s' \in \mathcal{S}} \in \mathbb{R}^{| \mathcal{S} |}$ and
the last inequality is from H\"{o}lder's inequality.
From Assumption~\ref{assumption: Lipschitz}, we have that 
\begin{align*}
    \| \big( \boldsymbol{f}( \boldsymbol{\mu}^* ) - \boldsymbol{f} ( \hat{\boldsymbol{\mu}}_k ) \big) (z) \|_1 
    &= \sum_{s' \in \mathcal{S}} | f^{s'}_{z} (\boldsymbol{\mu}^*) - f^{s'}_{z} (\hat{\boldsymbol{\mu}}_k) | \\
    &\leq L  \| \boldsymbol{\mu}^* - \hat{\boldsymbol{\mu}}_k \|_2 \sum_{s' \in \mathcal{S}} P(s' | z) \\
    &= L  \| \boldsymbol{\mu}^* - \hat{\boldsymbol{\mu}}_k \|_2.
\end{align*}
Therefore, applying Lemma~\ref{lem: mu bound} to \eqref{eqn: PV first term 1}, we can write
\begin{equation}
    \big[ P V^* - \mathbb{E}[ P_k V^* ] \big] (z)
    \leq \beta L \; \mathbb{E} \big[ \| \boldsymbol{\mu}^* - \hat{\boldsymbol{\mu}}_k \|_2 \big]
    \leq \frac{\beta L \sigma_{\boldsymbol{\mu}}}{\sqrt{n_k}}.
    \label{eqn: PV first term 2}
\end{equation}
The second term of \eqref{eqn: PV expand} can be bounded using \cite[Lemma~6]{gheshlaghi2013minimax}, i.e., with probability at least $1 - \delta$ we have
\begin{equation}
    \gamma \big[ \mathbb{E}[ P_k V^* ] - P_k V^* \big] (z)
    \leq c_{pv} \sqrt{ \hat{\sigma}_{\hat{V}^{\pi^*} } (z) } + \tilde{b}_{pv},
    \label{eqn: PV second term}
\end{equation}
where 
\begin{equation*}
    \tilde{b}_{pv} 
    = \bigg( \frac{ 5 (\gamma \beta)^{\frac{4}{3}} }{n_k} \log \frac{6N}{\delta} \bigg)^{\frac{3}{4}} 
    + \frac{3 \beta^2 }{n_k} \log \frac{12 N}{\delta}.
\end{equation*}
Substituting \eqref{eqn: PV first term 2} and \eqref{eqn: PV second term} into \eqref{eqn: PV expand} yields \eqref{eqn: PV upper bound}. Similarly, the inequality \eqref{eqn: PV lower bound} can be obtained by noting that 
$\big( \boldsymbol{f}( \boldsymbol{\mu}^* ) - \boldsymbol{f} ( \hat{\boldsymbol{\mu}}_k ) \big)^T V^* \geq - \| \boldsymbol{f}( \boldsymbol{\mu}^* ) - \boldsymbol{f} ( \hat{\boldsymbol{\mu}}_k) \|_1 \; \| V^* \|_{\infty}$ and applying \cite[Lemma~6]{gheshlaghi2013minimax}.
\end{proof}

Now, we are ready to provide our PAC sample complexity bound on the gap of the action-value functions $\| Q^* - Q^*_k \|$.

\begin{theorem}  \label{thm: Q error bound}
Suppose Assumptions~\ref{assumption: structural parameters}, \ref{assumption: subset contains info} and \ref{assumption: Lipschitz} hold. For all $\delta \in (0, 1)$ and $k \in \mathbb{N}$ such that $n_k > 0$, it holds that
\begin{equation*}
    \| Q^* - Q^*_k \| \leq \epsilon
\end{equation*}
with probability at least $1 - \delta$, where 
\begin{multline}
    \epsilon
    = \frac{\gamma \beta^2 L \sigma_{\boldsymbol{\mu}}}{\sqrt{n_k}}
    + \bigg( \frac{4 \beta^3}{n_k} \log \frac{4N}{\delta} \bigg)^{\frac{1}{2}} \\
    + \bigg( \frac{ 5 (\gamma \beta^2)^{\frac{4}{3}} }{n_k} \log \frac{12 N}{\delta} \bigg)^{\frac{3}{4}} 
    + \frac{3 \beta^3 }{n_k} \log \frac{24 N}{\delta}. 
    \label{eqn: epsilon thm}
\end{multline}
\end{theorem}

\begin{proof}
From \cite[Lemma~3]{gheshlaghi2013minimax}, we have
\begin{align}
    Q^* - Q_k^* \leq \gamma (I - \gamma P_k^{\pi^*})^{-1} [P - P_k] V^*, 
    \label{eqn: Q upper bound} \\
    Q^* - Q_k^* \geq \gamma (I - \gamma P_k^{\hat{\pi}^*_k})^{-1} [P - P_k] V^*. 
    \label{eqn: Q lower bound}
\end{align}
Substituting inequality \eqref{eqn: PV upper bound} in Lemma~\ref{lem: pv concentration} into \eqref{eqn: Q upper bound} to get that with probability at least $1 - \delta$,
\begin{equation*}
    Q^* - Q_k^* 
    \leq (I - \gamma P_k^{\pi^*})^{-1} \Big( c_{pv} \sqrt{ \hat{\sigma}_{\hat{V}^{\pi^*} }} + b_{pv} \boldsymbol{1} \Big).
\end{equation*}

Applying \cite[Lemma~8]{gheshlaghi2013minimax} and noting that $(I - \gamma P_k^{\pi^*})^{-1} \boldsymbol{1} = \beta \boldsymbol{1}$ yields
\begin{equation}
    Q^* - Q_k^* \leq ( \sqrt{2} \beta^{1.5} c_{pv} + \beta b_{pv} ) \boldsymbol{1},
    \label{eqn: Q upper bound 2}
\end{equation}
with probability at least $1 - \delta$. Similarly, from \eqref{eqn: Q lower bound}, we also have that 
\begin{equation}
    Q^* - Q_k^* \geq - ( \sqrt{2} \beta^{1.5} c_{pv} + \beta b_{pv} ) \boldsymbol{1},
    \label{eqn: Q lower bound 2}
\end{equation}
with probability at least $1 - \delta$. Result then follows by combining \eqref{eqn: Q upper bound 2} and \eqref{eqn: Q lower bound 2} using a union bound.
\end{proof}

\begin{remark}
The results provided above hold for general nonlinear functions $f^{s'}_z$ satisfying $\big( L P(s' | z) \big)$-Lipschitz continuity (Assumption~\ref{assumption: Lipschitz}). In particular, the term $\frac{\gamma \beta^2 L \sigma_{\boldsymbol{\mu}}}{\sqrt{n_k}}$ in Theorem~\ref{thm: Q error bound} comes from bounding the bias $f^{s'}_z (\boldsymbol{\mu}^*) - \mathbb{E} [ f^{s'}_z (\hat{\boldsymbol{\mu}}_k) ]$. In the special case where the reconstruction functions $f^{s'}_z$ are linear (e.g., directly estimating each entry of transition probabilities $P(s' | z)$ separately), the term $\frac{\gamma \beta^2 L \sigma_{\boldsymbol{\mu}}}{\sqrt{n_k}}$ disappears (since in the proof, $\mathbb{E} [ f^{s'}_z (\hat{\boldsymbol{\mu}}_k) ] = f^{s'}_z ( \mathbb{E} [ \hat{\boldsymbol{\mu}}_k ] ) = f^{s'}_z (\boldsymbol{\mu}^*)$) and we can recover the results in \cite{gheshlaghi2013minimax}.
\end{remark}

\begin{remark}
Although Theorem~\ref{thm: Q error bound} here is similar to \cite{gheshlaghi2013minimax}, it is crucial to emphasize that the quantity $n_k$ here is the minimum amount of information received about each structural parameter $\mu_i$ up to time-step $k$, which can be significantly larger than the minimum number of visits over all state-action pairs (i.e., $n$ in \cite{gheshlaghi2013minimax}) as in traditional methods that estimate each element of the transition matrix.
\end{remark}

Based on the result of Theorem~\ref{thm: Q error bound}, we derive a sample complexity bound in terms of $n_k$, i.e., the least amount of samples we have received about any structural parameter up to time-step $k$. 

\begin{corollary}  \label{cor: no. of samples}
Suppose Assumptions~\ref{assumption: structural parameters}, \ref{assumption: subset contains info} and \ref{assumption: Lipschitz} hold.
For all $\delta, \epsilon \in (0, 1)$, the number of samples for each structural parameter $n_k$ given below suffices for the uniform approximation error $\| Q^* - Q^*_k \| \leq \epsilon$ with probability at least $1 - \delta$.
\begin{enumerate}
    \item If $m = \mathcal{O} (\sqrt{\log N})$, then $n_k = \mathcal{O} \big( \frac{\beta^4 L^2}{\epsilon^2} \log \frac{N}{\delta} \big)$.
    
    \item If $m = \mathcal{O}\big( \log N \big)$, then $n_k = \mathcal{O} \big( \frac{\beta^4 L^2}{\epsilon^2} \big( \log \frac{N}{\delta} \big)^2 \big)$.
\end{enumerate}
\end{corollary}

\begin{proof}
From Theorem~\ref{thm: Q error bound}, the quantity $\epsilon$ in \eqref{eqn: epsilon thm} can be bounded as follows. There exists a constant $C \in \mathbb{R}$ such that
\begin{equation*}
    \epsilon 
    \leq \frac{\gamma \beta^2 L \sigma_{\boldsymbol{\mu}}}{\sqrt{n_k}} + \bigg( \frac{C \beta^3}{n_k} \log \frac{N}{\delta} \bigg)^{\frac{1}{2}}.
\end{equation*}
The above inequality implies that
\begin{multline}
    \epsilon^2
    \leq \frac{1}{n_k} \bigg( \gamma^2 \beta^4 L^2 \sigma_{\boldsymbol{\mu}}^2 
    + \gamma \beta^{3.5} L \sigma_{\boldsymbol{\mu}} \sqrt{C \log \frac{N}{\delta}} \\
    + C \beta^3 \log \frac{N}{\delta} \bigg).
    \label{eqn: epsilon square}
\end{multline}

Consider the first case where $m = \mathcal{O} ( \sqrt{\log N} )$. We have that 
\begin{equation*}
    \sigma_{\boldsymbol{\mu}} = \sum_{i=1}^m \sigma_{\mu_i} = \mathcal{O} ( \sqrt{\log N} ),
\end{equation*}
and thus, inequality \eqref{eqn: epsilon square} implies that $n_k = \mathcal{O} \big( \frac{\beta^4 L^2}{\epsilon^2} \log \frac{N}{\delta} \big)$.

Consider the second case where $m = \mathcal{O} ( \log N )$, we have that
\begin{equation*}
    \sigma_{\boldsymbol{\mu}} = \sum_{i=1}^m \sigma_{\mu_i} = \mathcal{O} ( \log N ),
\end{equation*}
and thus, inequality \eqref{eqn: epsilon square} implies that $n_k = \mathcal{O} \big( \frac{\beta^4 L^2}{\epsilon^2} \big( \log \frac{N}{\delta} \big)^2 \big)$.
\end{proof}

It is noted that the PAC sample complexity bound in Corollary~\ref{cor: no. of samples} is for $n_k$ instead of the total number of time-steps $k$. Specifically, by using the same sample generation technique (generative model) as in \cite{azar2011reinforcement}, we can compare our results with the previous classic bound on the total number of time-steps, $k = \mathcal{O}\big( \frac{N \beta^3}{\epsilon^2} \log \frac{N}{\delta} \big)$, derived in \cite{azar2011reinforcement}. Under the generative model, at each time-step $k$, each state-action pair has $\big\lfloor \frac{k}{N} \big\rfloor$ visitations. From Assumption~\ref{assumption: subset contains info}, recall that $\mathcal{U}_i$ is the subset of all state-action pairs that provides information about $\mu_i^*$. Then, the minimum amount of information for any structural parameter at time-step $k$ is 
\begin{equation*}
    n_k = \min_{i \in \{ 1, 2, \ldots, m \}} | \mathcal{U}_i | \times \bigg\lfloor \frac{k}{N} \bigg\rfloor.
\end{equation*}
From Corollary~\ref{cor: no. of samples}, we have that both of the following conditions suffice for achieving the classic bound $k = \mathcal{O}\big( \frac{N \beta^3}{\epsilon^2} \log \frac{N}{\delta} \big)$.
\begin{enumerate}
    \item $m = \mathcal{O} (\sqrt{\log N})$ and $| \mathcal{U}_i | = \Omega ( \beta L^2 )$ for all $i \in \{ 1, 2, \ldots, m \}$.
    
    \item $m = \mathcal{O}\big( \log N \big)$ and $| \mathcal{U}_i | = \Omega \big( \beta L^2 \log \frac{N}{\delta} \big)$  for all $i \in \{ 1, 2, \ldots, m \}$.
\end{enumerate}

As we will later show in the illustrative example and numerical experiments, in many practical scenarios, $|\mathcal{U}_i| = \Theta (N)$, for all $i = \{1,2,\ldots,m\}$, which therefore indicates better sample efficiency compared to the classic result.

\section{An Illustrative Example of the Structural Estimation-based Approach}
\label{section:An Application to Queuing System}

In this section, we describe a practical scenario to illustrate the structural estimation-based planning. We use a general class of queuing system, i.e., discrete-time Geo/Geo/k model (based on Kendall's notation), as an instance. 

Assume that the system consists of a stream of (identical) tasks arriving to a queue buffer, following a geometric distribution with parameter $I^*$ as the true injection rate (i.e., at each time-step, there is a probability $I^*$ that a new job arrives to the queue).  The buffer has a finite size, $B$, which is served by $G$ servers with different exit probabilities (i.e., different processing speeds), $\mu_1^*, \ldots, \mu_G^*$. The exit of task at server $i \in \{ 1, 2, \ldots, G \}$ follows a Bernoulli distribution with parameter $\mu_i^*$ (i.e., the service times follow a geometric distribution, and at each time-step, if there is a task on server $i$, that task is completed with probability $\mu_i^*$). Our aim is to minimize the discounted infinite horizon sum of the jobs remaining in the system (i.e., in both queue and servers) over all time steps.

This system can be viewed as a MDP $M= \langle \mathcal{S},\mathcal{A},P_M,R,\gamma \rangle$. Specifically, the state of the system at time-step $k$ can be defined as a vector $(l_k, s_{k,1}, \ldots, s_{k,G}) \in \mathbb{R}^{G+1}$ where $l_k \in \{ 0, 1, \ldots, B \}$ indicates the number of tasks remaining in the queue, and the binary variable $s_{k,i}$ for $i \in \{ 1, \ldots, G \}$, indicates the operation status of server $i$, i.e.,  $s_{k,i}=1$ indicates that there is a task being processed by server $i$, while $s_{k,i}=0$ indicates that the server $i$ is available for task assignment. The action at time-step $k$ can be defined as a binary vector $(a_{k,1},\ldots,a_{k,G})\in \mathbb{R}^G$, where for $i \in \{ 1, \ldots, G \}$, $a_{k,i}=1$ indicates the assignment of task to server $i$, and $a_{k,i}=0$ otherwise. The reward function can be chosen as the negative of the total number of jobs in the system, i.e., $r_k=-(l_k+\sum_{i=1}^G s_{k,i})$. It is noted that $|\mathcal{S}|= (B+1) 2^{G}$ and $|\mathcal{A}|=2^{G}$. Thus, we have $N=|\mathcal{S}||\mathcal{A}|=(B+1)2^{2G}$.

Suppose we do not know the true values of the arrival and service rates $I^*, \mu_1^*, \ldots, \mu_G^*$ a priori; however, given that we know this is a queuing system, the structure of the system is known (i.e., each element of the probability transition matrix is a polynomial in those parameters for this given system, where the highest degree of any polynomial is $G+1$).  If one treated the system as a black box (i.e., with no structural knowledge), there would be the order of $|\mathcal{S}|^2|\mathcal{A}|$ elements to estimate (i.e., each element of the transition matrix). However, by leveraging structural information, there are only $m=G+1$ parameters to estimate, from which we can derive the estimated transition matrix $P_{k}$ based on $k$ data samples. Note that $m = \frac{1}{2} \log_2 \big( \frac{N}{B+1} \big) + 1 = \mathcal{O} (\log N)$. This corresponds to the second case in Corollary~\ref{cor: no. of samples}.

At each time-step $k$, we implement MLE for the unbiased estimation of the injection rate and exit probabilities, i.e., $\hat{I}_{k}=\frac{1}{n_{k,I}} \sum_{t \in T_{k,I}} Y_t$, ${{\mu}}_{i,k}=\frac{1}{n_{k,i}} \sum_{t \in T_{k,i}} X_{t,i}$, where $Y_t \sim {\rm Bern} \{I^*\}$ and $X_{i,t} \sim {\rm Bern} \{\mu_i^*\}$ are random variables for the number of tasks remaining in the queue and $i$-th server's operation status, and $T_{k,I}$, $T_{k,i}$ are the corresponding sets of time-steps.

The first example shows how we extract structural information from the collected samples and explains how to determine $\mathcal{U}_i$, $i \in \{ 1, 2, 3 \}$, as specified in Assumption~\ref{assumption: subset contains info}.
\begin{example} \label{example: binary rv}
At time-step $k$, suppose we collect a sample $(s_k,a_k,s_{k+1})$, where $s_k=(l_k=3,s_{k,1}=0,s_{k,2}=0,s_{k,3}=0)$, $a_k=(a_{k,1}=0,a_{k,2}=1,a_{k,3}=1)$, and $s_{k+1}=(l_k=2,s_{k,1}=0,s_{k,2}=0,s_{k,3}=1)$. Then, we can derive that $Y_k = 1$, i.e., a job arrival happens at time $k$, and $n_{k,I}$ increases by 1; $X_{k,2}=1$, i.e., the job departs from server 2, and $X_{k,3}=0$, i.e., the job fails to depart from server 3. Both of $n_{k,2}$ and $n_{k,3}$ increase by 1. $ \square $
\end{example} 

Let $\mathcal{U}_0$ be the subset corresponding to $I^*$, and $\mathcal{U}_i$ be the subset corresponding to $\mu_i^*$ for $i \in \{ 1, 2, 3 \}$. Given state-action pair $(s_k, a_k)$ as the sample discussed in Example~\ref{example: binary rv}, no matter which next state $s_{k+1}$ is, the corresponding transitions provide the information of all the structures except $\mu_1$. In general, $\mathcal{U}_0$ contains all of the state-action pairs except for the cases when the queue is full, i.e., $l_k = B$. On the other hand, $\mathcal{U}_i$ contains the state-action pairs where either $s_{k,i} = 1$, or $s_{k,i} = 0$ and $a_{k,i} = 1$. Thus, we have $|\mathcal{U}_j| = \Theta (|\mathcal{S}||\mathcal{A}|) = \Theta (N)$ for all $j \in \{ 0, 1, 2, 3 \}$, which implies better sample efficiency in terms of the classic result, as discussed in Section~\ref{section:Estimation via Structural Parameters and Evaluation Error Analysis}.

The second example will show how each entry of the transition matrix can be expressed as a polynomial of the structural parameters and entries of a given policy vector, which therefore satisfies the Assumption~\ref{assumption: structural parameters}.

\begin{example}
Consider a current state $S = (l=1, s_1=0, s_2=0, s_3=0)$, an action $A = (a_1=1, a_2=0, a_3=0)$, and a next state $S' = (l=0, s_1=0, s_2=0, s_3=0)$. We have that $P(S' | S, A) = (1 - I^*)\mu_1^*$, which indicates the probability that two events, i.e., a job does not arrive, and a job departs from server 1, happen at the same time-step. $ \square $
\end{example}

\section{Numerical Experiments}

In this section, we implement two numerical experiments in two different environments. In these experiments, we consider a two-stage framework: the online exploration stage and the offline optimization stage. At the online exploration stage, we aim to collect samples for estimating the transition probability matrix. More specifically, we first derive the estimates of structural parameters, based on which we estimate the $P$ matrix. At the offline optimization stage, we adopt the policy iteration (PI) algorithm. For comparison, we also consider the entry-wise estimation-based approach and a model-free approach (Q-learning). The entry-wise estimation-based approach belongs to the model-based method where we apply maximum likelihood estimation (MLE) to estimate $P$ matrix in the entry-wise fashion directly from the collected samples. It is noted that the entry-wise approach can be viewed as a special case of structure estimation-based approach in the sense that each entry is treated as a structural parameter. In particular, for the entry-wise case, from \eqref{def: min_samples}, we have 
\begin{equation*}
    n_k = \min_{(s,a) \in \mathcal{S} \times \mathcal{A}} n_{k}(s,a), 
\end{equation*} 
where $n_{k}(s,a)$ denotes the number of tuples $(s,a)$ within $\mathcal{D}_k$.

Following the mindset of theoretical analysis in section~\ref{section:Estimation via Structural Parameters and Evaluation Error Analysis}, we mainly focus on sample efficiency of the Q estimation error ($\| Q^* - Q_k^* \|$), and the ratio of the minimum amount of information $n_k$ to the total sample budget $k$.

\subsection{Queuing network}

For the queuing network environment, we consider a discrete-time Geo/Geo/3 queuing model, where the queue length is $B=8$, the injection rate is $I^*=0.85$, and the exit probabilities for the three servers are $\mu_1^*=0.9$, $\mu_2^*=0.01$, and $\mu_3^*=0.04$, respectively. For the corresponding MDP, we set $\gamma=0.9$, and have $|\mathcal{S}|=72$, $|\mathcal{A}|=8$.

The results are shown in Figures~\ref{figure: evaluation error queue} and \ref{figure: valid sample queue}. From Figure~\ref{figure: evaluation error queue}, we observe that the structure-based approach needs only $2 \times 10^4$ samples for convergence to $Q^*$, while the entry-wise approach requires more than $3.5 \times 10^7$ samples. Figure~\ref{figure: valid sample queue} indicates that, in the structural case, almost every collected sample can provide structural information since the ratio is close to $1$. In contrast, in the entry-wise case, due to the large amount of $P$ matrix entries (i.e., $|\mathcal{S}|^2 |\mathcal{A}|$), the minimum amount of information increases very slowly with the total number of samples. This shows the potential of leveraging structural information to dramatically increase sample efficiency.

\begin{figure}[htbp]
\centering
\begin{minipage}[t]{0.47\textwidth}
\centering
\includegraphics[width=7.0cm]{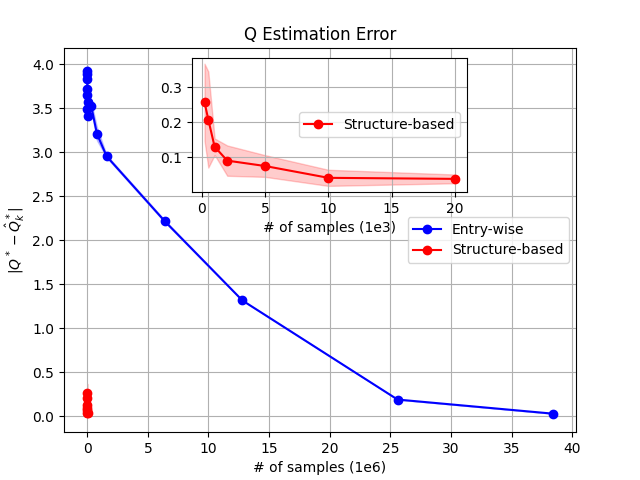}
\caption{Q error in queuing network}
\label{figure: evaluation error queue}
\end{minipage}
\begin{minipage}[t]{0.47\textwidth}
\centering
\includegraphics[width=7.0cm]{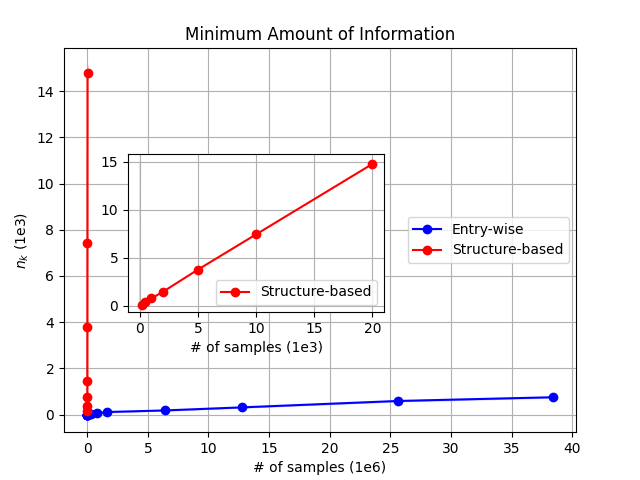}
\caption{Minimum info in queuing network}
\label{figure: valid sample queue}
\end{minipage}
\end{figure}

\subsection{Stochastic windy gridworld}  \label{subsec: stochastic windy}

Windy gridworld, as shown in Figure~\ref{figure: Stochastic windy gridworld}, is a classical environment first used in \cite[Example~6.5]{sutton2018reinforcement}. It is a standard gridworld with a crosswind running upward through the grids of certain columns. The strength of the wind (as denoted below each column) would make the resultant next state be shifted upward by the corresponding distance (e.g., wind strength 1 would result in one grid shifted upward). The stochastic windy gridworld we consider is a complicated variant of the classical version where the wind and state transitions have stochasticity, i.e., with a probability of $\beta_i$, the wind of $i$-th column would affect the dynamics of the agent according to its strength, and with a probability of $\alpha$, the agent will transit to one of the surrounding states according to an uniform distribution independent of agent's action. This environment is originated from \cite{de2018multi} and \cite{dumke2017double}. The state of the system is a 2-dimensional vector $(x,y)$ using the agent's location in terms of x-axis and y-axis (ranging from $0 \sim 9$ and $0 \sim 6$ respectively). At each time-step, an action is chosen from $(0,1)$ (up), $(0,-1)$ (down), $(-1,0)$ (left), and $(0,1)$ (right). Thus, we have $|\mathcal{S}| = 70$ and $|\mathcal{A}| = 4$. For other parameters of the environment, we set discount factor $\gamma=0.9$, $\alpha=0.4$, and uniform stochasticity parameter for all winds, $\beta_0 = \beta_1 = \cdots = \beta_9 = 0.5$.

It is noted that the agent is blocked by the borderline, e.g., implementing either action $(-1,0)$ or $(0,-1)$ at state $(0,0)$ (the grid at bottom left) would not result in any movement. There are constant rewards of $-1$ until the goal state is reached. The blue line in Figure \ref{figure: Stochastic windy gridworld} indicates an optimal trajectory for the \textit{deterministic version} given the start state, $S$ (i.e., state $(0,3)$), and goal state, $G$ (i.e., state $(7,3)$).

Satisfying Assumption~\ref{assumption: subset contains info}, the structural parameters, in this case, consist of the 10 probabilities of each crosswind and the probability of taking uniformly random action. These are associated with 11 random variables (i.e., $X_0, X_1, \ldots, X_{10}$) following Bernoulli distribution with parameters $\beta_0$, $\beta_1$, $\ldots$, $\beta_{9}$ and $\alpha$, respectively.\footnote{Although the strength of wind is initially unknown to us and determines the dynamics of the system, we do not treat it as a structural parameter here since it is deterministic and could be easily inferred using single sample (transition).} Specifically, for $X_3$, i.e., the probability of the crosswind in the 3-rd column (i.e., $x=3$), its corresponding subset $\mathcal{U}_3$ contains the state-action pairs satisfying either (1) $(x=3, 0 \leq y \leq 4)$, (2) $(x=3, y=5)$, $a \in \{(-1,0),(1,0),(0,-1)\}$, or (3) $(x=3, y=6)$, $a \in \{(0,-1)\}$. In general, the information of structure $\beta_i$ could only be provided by the state-action pairs where the states belong to its own column (i.e., $i$-th column). For $X_{10}$, i.e., the probability of taking uniformly random action, its corresponding subset $\mathcal{U}_{10}$ contains all state-action pairs. Thus, we also have $|\mathcal{U}_i| = \Theta (|\mathcal{S}||\mathcal{A}|) = \Theta (N)$, which implies better sample efficiency in terms of the classic result, as discussed in Section~\ref{section:Estimation via Structural Parameters and Evaluation Error Analysis}.

Assuming that we are able to increase the size of state space by increasing the number of columns but the set of possible wind strengths and associated stochasticity is fixed, i.e., it does not depend on the number of columns, in this case, we have that $m = \mathcal{O}(1)$. This corresponds to the first case in Corollary~\ref{cor: no. of samples}.

\begin{figure}[thpb]
  \centering
  \includegraphics[scale=0.35]{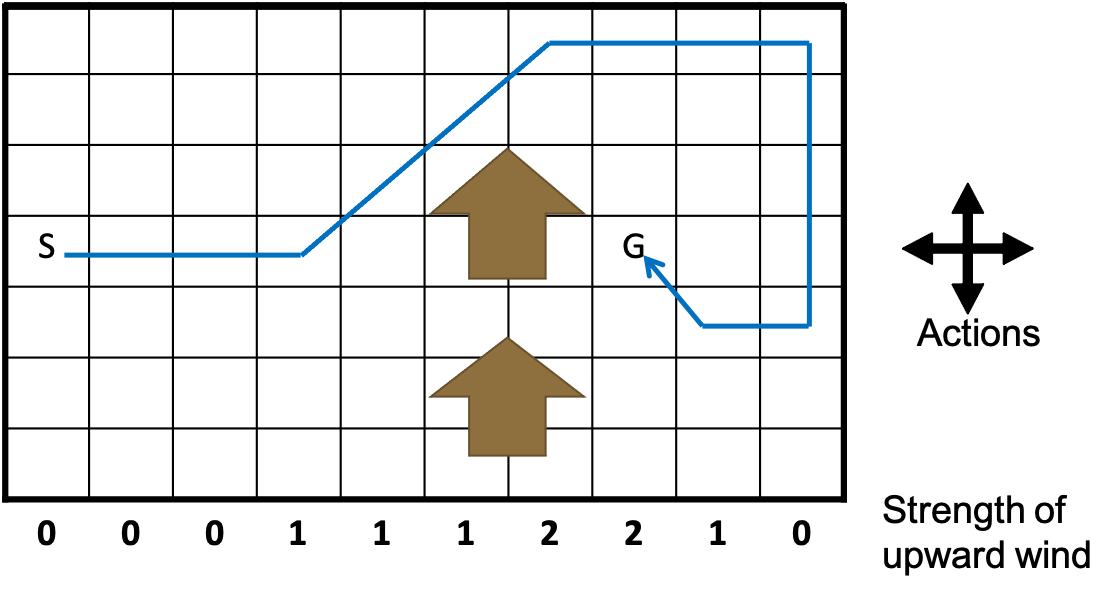}
  \caption{Stochastic windy gridworld}
  \label{figure: Stochastic windy gridworld}
\end{figure}

We investigate both entry-wise estimation (black box) and structural estimation (grey box) for model-based method. For the structural case, we further consider the different levels of structural information being used for reconstructing the $P$ matrix. Specifically, we consider two cases: the case of owning more structural information (``More Info'') and least structural information (``Lst Info'').  In the case of ``More Info'', we know the specific columns (i.e., 3-rd, 4-th, 5-th and 8-th columns) with the same wind strength and associated probabilities, and the columns (0-th, 1-st, 2-nd, and 9-th columns) that are not affected by the wind. Thus, only three structural parameters are needed. 

In the case of ``Lst Info'', we only know that each column might be affected by the stochastic upward wind. In this case, 11 structural parameters are needed. Also, we implement a model-free method, Q-learning, for comparison.

The results for model-based approaches are shown in Figures~\ref{figure: Evaluation_error_SW} and~\ref{figure: valid sample in windy gridworld}. In order to achieve sufficiently small Q estimation error (e.g., $\| Q^* - Q_k^* \| < 0.01$), the entry-wise case needs $2 \times 10^7$ samples, while for structural cases, ``Lst Info'' needs $2 \times 10^4$, and ``More Info'' needs only $2.5 \times 10^3$ samples. This shows the power of structural information in terms of sample efficiency, which can be reflected from the relationship between $n_k$ and $k$ in Figure~\ref{figure: valid sample in windy gridworld} as well. The values of ${n_k}/{k}$ are approximately $4\%$, $1.1\%$, $0.28\%$ on average for ``More Info'', ``Lst Info'' and entry-wise cases, respectively.


\begin{figure}[thpb]
  \centering
  \includegraphics[scale=0.5]{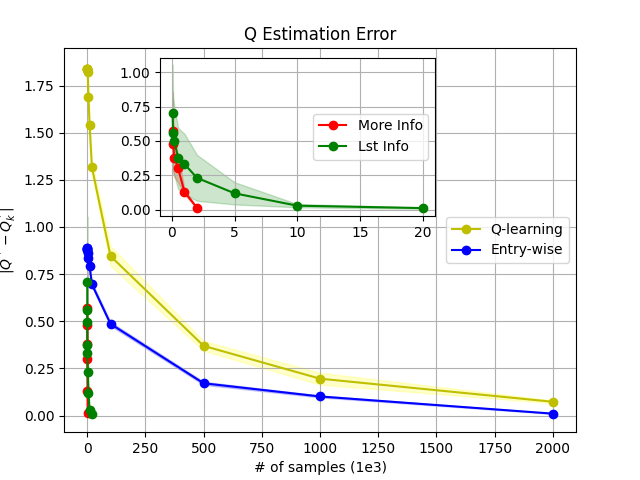}
  \caption{Q error in gridworld}
  \label{figure: Evaluation_error_SW}
\end{figure}

\begin{figure}[thpb]
  \centering
  \includegraphics[scale=0.5]{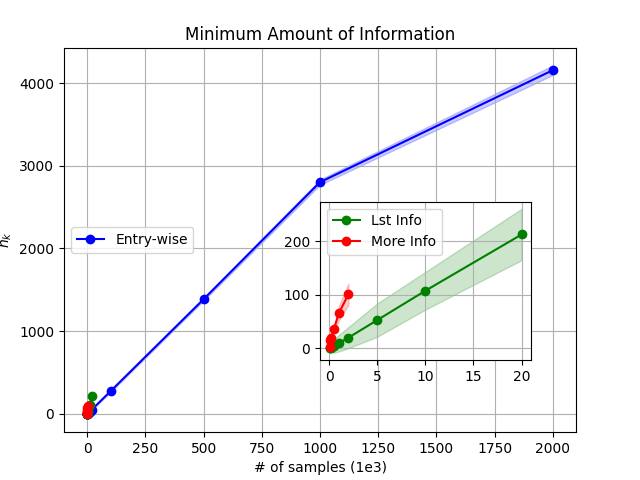}
  \caption{Minimum Info in gridworld}
  \label{figure: valid sample in windy gridworld}
\end{figure}

The results of the model-free Q-learning method for comparison are also shown in Figure~\ref{figure: Evaluation_error_SW}. According to the results, we find that the model-free method shows similar performance with entry-wise approach. This shows that the model-based approach may not always be more sample efficient than a model-free method especially when there are many structures to be estimated.

\section{Conclusion}

In this paper, we consider leveraging the power of structural knowledge about transition probability matrix in a model-based reinforcement learning problem. Specifically, we identify the explicit relationships between structural parameter and model estimation error, model estimation and evaluation error. In terms of the sample efficiency, we not only provide the theoretical results on PAC sample complexity bound on the action-value function, but also we empirically show the advantage of leveraging structure information.

Our proposed structure estimation-based approach is general and could be applied to various problems. Besides, our approach could be further applied to the case where structural parameters are time-varying. For example, the wind strength in stochastic windy gridworld varies along with time, i.e., the underlying MDP is non-stationary, which is one of our future steps from an experimental perspective.

The major limitation of the current approach comes from the heavy dependence on specific knowledge of system dynamics (e.g., the wind dynamics in stochastic windy GridWorld). In a practical scenario, it is highly non-trivial to discover the structural parameters for large-scale problems in an automatic and general way. Although, there exists some research (e.g., HiP-MDPs, \cite{doshi2016hidden}, \cite{perez2020generalized}, \cite{killian2017robust}) working towards this topic, the generated structural properties and mapping functions would not satisfy assumptions we made and thus cannot guarantee convergence and finite-time performance of the algorithms. Having said that, we can still leverage specific knowledge of a given system dynamics, such as model approximation methods in queuing theory, to reduce the dimensionality of the model and allow the generation of small-size structural parameters.





\bibliographystyle{IEEEtran}
\bibliography{Final_arXiv_version}

\end{document}